\documentclass[11pt,a4paper]{article}
\usepackage[margin=3cm]{geometry}
\usepackage{authblk}

\usepackage[T1]{fontenc}    %
\usepackage{microtype}      %
\usepackage{booktabs}
\usepackage{xcolor}
\usepackage{lmodern}

\usepackage{mathtools,amssymb}   %
\mathtoolsset{showonlyrefs,showmanualtags}

\usepackage{mathtools,amssymb}
\usepackage{my_notation}

\usepackage{amsthm,thmtools}

\usepackage{csquotes}
\usepackage[british]{babel}     %

\usepackage{my_citations}   %
\usepackage[hidelinks]{hyperref}
\usepackage{zref-clever}
\let\cref\zcref
\zcsetup{cap,nameinlink=false}                   %

\AddToHook{env/algorithm/begin}{%
   \zcsetup{reftype=algorithm}}

\zcRefTypeSetup{algorithm}{
Name-sg = Algorithm,
name-sg = algorithm,
Name-pl = Algorithms,
name-pl = algorithms,
}
\usepackage{my_theorems}
\usepackage{url}

\usepackage{enumitem}           %

\title{\textbf{Variance-sensitive Thompson sampling\\for generalised linear bandits, revisited}}

\author[1]{Tom Perneczky}
\author[2]{Marc Abeille}
\author[1]{David Janz}
\affil[1]{University of Oxford}
\affil[2]{Criteo AI Lab}

\date{}

\begin{document}
\maketitle

\begin{abstract}
  We prove a variance-sensitive regret bound for Thompson sampling in stochastic generalised linear bandits. The argument assumes a warm-up, after which the regret is controlled through using the Gaussian Poincar\'e inequality. This bypasses the point at which previous optimism-based analyses break down. Removing the warm-up while retaining the same variance-sensitive scaling remains open, and appears nontrivial.
\end{abstract}

\section{Introduction}

Thompson sampling is a classic Bayesian-inspired alternative to upper-confidence bound algorithms for balancing exploration and exploitation in sequential decision-making problems \citep{thompson1933likelihood,scott2010modern,russo2018tutorial}. The method yields promising empirical performance \citep{chapelle2011empirical}, sometimes with substantially better computational complexity, and allowing for natural parallelisation \citep{hernandezlobato2017parallel,kandasamy2018parallelised}. However, while the procedure admits clean guarantees under Bayesian assumptions, its frequentist analysis is complex, even in the simplest unstructured settings \citep{agrawal2012analysis}. Our understanding of Thompson sampling in structured settings is still evolving \citep{abeille2025randomised}.

This manuscript examines the regret of Thompson sampling when interacting with generalised linear bandits \citep{filippi2010parametric}, a setting that includes logistic and Poisson bandits; the results also imply simple first-order bounds for linear bandits with bounded rewards. Early Thompson sampling regret bounds for generalised linear bandits were developed by \citep[Appendix F]{abeille2017linear} and \citet{kveton2020randomized}. Later bounds developed in \citet[Appendix~D.2]{faury2022jointly} and \citet{janz2024evill} promise performance that adapts to the variance of the optimal arm; both of the variance-sensitive proof attempts are, however, flawed.

We rebuild the results of \citet{faury2022jointly,janz2024evill}, pivoting to proofs inspired by the new approach to linear bandits of \citet{abeille2025randomised}. In particular, we use the Gaussian Poincar\'e  inequality \citep{ledoux2001concentration} to control the expected per-step regret, which allows us to successfully recover variance-sensitive regret.

Our proof does, however, rely on a warm start---that is, an a~priori guarantee that we start with a reasonable estimate of the optimum, which may be achieved through algorithmic means, or may be based on prior data. Such warm start conditions were once common in optimism-based algorithms \citep{faury2020improved,russac2021selfconcordant,faury2022jointly}, but were later decisively eliminated \citep{lee2024unified,akhavan2025bernstein}.

Recent work of \citet{liu2026efficient} attains variance-sensitive `simple regret' guarantee with a Thompson sampling-like algorithm in the contextual logistic bandit setting.

\section{Problem setting}

The bandit setting we consider is defined in terms of a natural exponential family.

\paragraph{Natural exponential family.}
Let $Q$ be a non-degenerate probability distribution on $\R$, and let
\[
  \psi_Q(u)=\log\int e^{uy}Q(\dif y)
\]
be its cumulant generating function. Let $\cU\subset\R$ be the nonempty open interval on which $\psi_Q$ is finite. Consider the natural exponential family
\[
  P_u(\dif y)=\exp\roundb[\big]{uy-\psi_Q(u)}Q(\dif y),\qquad u\in\cU.
\]
Write
\[
  \mu(u)=\psi_Q'(u),\qquad \nu(u)=\psi_Q''(u)=\mu'(u),\qquad u\in\cU.
\]
Then $P_u$ has mean $\mu(u)$ and variance $\nu(u)$. Since $Q$ is non-degenerate, $\nu(u)>0$ on $\cU$, and $\mu$ is strictly increasing.

\paragraph{Bandit environment.} Fix a natural exponential family with notation as above. Let $b > 0$ with $[-b, b] \subset \cU$ and fix an unknown parameter $\theta_\star\in\Rd$ with $\norm{\theta_\star}_2\le b$.

The interaction proceeds at step $t=1,2,3,\dots$ with the learner selecting an arm $X_t$ in a compact action set $\cX\subset\Bd$, and observing
\[
  Y_t \sim P_{\langle X_t,\theta_\star\rangle}.
\]
We judge the learner's performance by its incurred $n$-step (pseudo-)regret
\[
  R_n
  =
  \sum_{t=1}^n
  \roundb[\big]{\mu(\langle x_\star,\theta_\star\rangle)-\mu(\langle X_t,\theta_\star\rangle)},
\]
where $x_\star \in \argmax_{x \in \cX} \langle x, \theta_\star\rangle$. While here, we also define
\[
\nu_\star = \nu(\langle \theta_\star, x_\star\rangle)\,.
\]
We seek high-probability upper bounds on $R_n$ with the leading terms scaling as $\sqrt{\nu_\star}$.

\paragraph{Problem-specific quantities.} Let $B > b$ be such that $[-B,B] \subset \cU$; since $[-b,b] \subset \cU$ and $\cU$ is open, such a $B$ exists. Define the following quantities on this expanded interval:
\[
  M = \sup_{|u| \leq B}\frac{|\psi_Q'''(u)|}{\psi_Q''(u)}\,, \qquad L=\sup_{|u|\le B}\nu(u)\vee1,\qquad
  \kappa= \sup_{|u|\le B} \frac{1}{\nu(u)}\vee1.
\]
The above quantities will feature in the algorithm and analysis; therefore, both depend on the choice of $B$. For simplicity, however, we do not optimise over $B$.

\section{Algorithm and main result}
\paragraph{Confidence sequence.}
Let $(\cF_t)_{t\ge0}$ be the data filtration,
\[
  \cF_t=\sigma(X_1,Y_1,\dots,X_t,Y_t),
\]
with $\cF_0$ trivial. The regret analysis below only uses the following statistical input.

\begin{assumption}[Confidence sequence]\label{ass:cs}
Fix $\delta\in(0,1/4)$ and $\lambda\ge1$. Let $(\hat\theta_t)_{t\ge0}$ be an $(\cF_t)$-adapted sequence in $\Rd$, and let
$(\beta_t^\delta)_{t\ge0}$ be a positive nondecreasing $(\cF_t)$-adapted sequence. Assume that there exists a deterministic nondecreasing sequence of real numbers
$(\bar\beta_t^\delta)_{t\ge0}$ such that $\beta_t^\delta\le\bar{\beta}_t^\delta$ almost surely for every $t\ge0$, and that
\[
  \cE^\star
  =
  \curlyb[\Big]{
    \forall\, t\in[n],\
    \norm{\hat\theta_{t-1}-\theta_\star}_{H_{t-1}^\star}
    \le \beta_{t-1}^\delta} \spaced{satisfies} \P{\cE^\star}\ge1-\delta,
\]
where
\[
  H_t^\star
  =
  \lambda I+\sum_{s=1}^t
  \nu(\langle X_s,\theta_\star\rangle)X_sX_s\tran .
\]
\end{assumption}

\begin{remark}
  To satisfy \cref{ass:cs}, we might for example take the unconstrained regularised maximum-likelihood estimator
\[
  \hat\theta_t\in\argmin_{\theta\in\Rd}
  \curlyb[\Bigg]{
    \frac{\lambda}{2}\norm{\theta}_2^2
    -\sum_{s=1}^t\roundb[\Big]{Y_s\langle\theta,X_s\rangle-\psi_Q(\langle\theta,X_s\rangle)}
  },
\]
where $\psi_Q$ is understood to be $+\infty$ outside $\cU$. The corresponding choice of $(\beta_t)_{t \geq 0}$ and $(\bar\beta_t)_{t \geq 0}$ is then given by standard results in the literature \citep[see, e.g.,][]{akhavan2025bernstein}.
\end{remark}

\paragraph{Algorithm.} The algorithm consists of a deterministic warm-up and a main phase:
\begin{enumerate}
  \item \emph{Warm-up.} Fix $\epsilon > 0$. Starting from $V_0=\lambda I$, while
$\max_{x\in\cX}\norm{x}_{V_t^{-1}}^2>\epsilon$ and $t<n$, the algorithm chooses
\[
  X_{t+1}\in\argmax_{x\in\cX}\norm{x}_{V_t^{-1}},
  \qquad
  V_{t+1}=V_t+X_{t+1}X_{t+1}\tran.
\]
Write $\tau$ for the index of the final covariate chosen by the warm-up and set
\[
  G_\tau
  =
  \lambda I+\frac1\kappa\sum_{s=1}^{\tau}X_sX_s\tran.
\]
\item \emph{Thompson sampling.} Let $(\eta_t)_{t=\tau+1}^n$ be independent $\cN(0,I_d)$ random vectors. For $t=\tau+1,\dots,n$, take
\[
  \theta_t=\hat\theta_{t-1}+\gamma\beta_{t-1}^\delta G_{t-1}^{-1/2}\eta_t, \spaced{choose} X_t\in\argmax_{x\in\cX}\langle\theta_t,x\rangle,
\]
and update
\[
  G_t=G_{t-1}+\tilde\nu(\langle\hat\theta_{t-1},X_t\rangle) X_tX_t\tran, \spaced{where} \tilde\nu(u) = \nu((-B) \vee u \wedge B).
\]
\end{enumerate}

\paragraph{Main result.} The following theorem captures the main result of this manuscript.

\newcommand{\epsloc}{\epsilon_{\rm loc}}

\begin{theorem}[Regret bound]\label{thm:glm-gaussian-sampling-regret}
Fix $\delta\in(0,1/4)$, $\lambda\ge1$, $n \geq 1$, and suppose \cref{ass:cs} holds. Define
\[
  r_{\rm loc}:=\frac{B-b}{2}\wedge\frac1M,
  \qquad
  \Gamma_\delta=\sqrt d+\sqrt{2\log(n/\delta)},
  \qquad
  \ell_n=\log\roundb[\big]{1+\frac{nL}{\lambda d}},
\]
where $1/M = \infty$ if $M=0$. Suppose we run the algorithm above with $\gamma = 4$ and warm-up parameter
\[
  \epsloc
  =
  \frac{1}{\kappa L}
  \wedge
  \frac{r_{\rm loc}^2}{4\kappa(\bar{\beta}_{n-1}^\delta)^2 (1 \vee \gamma\Gamma_\delta)^2},
\]
where the second term in the minimum is omitted if $\bar\beta_{n-1}^\delta=0$.
Then, with probability at least $1-4\delta$,
\[
  R_n
  \le
  \Delta\tau
  +300\,\beta_{n-1}^\delta\Gamma_\delta
  \sqrt{\nu_\star nd\,\ell_n}
  +80\,\nu_\star
  \roundb[\Big]{\frac{\beta_{n-1}^\delta}{\sqrt{\lambda}}+r_{\rm loc}}
  \sqrt{n\log(1/\delta)},
\]
where $\Delta = \sup_{u,u' \in [-b,b]} \mu(u)-\mu(u')$. Moreover, the length of the warm-up satisfies
\[
\tau \leq
    \frac{4d}{\epsloc}
    \log\roundb[\big]{1+\frac{4d}{\epsloc\lambda}}
    +1.
\]
\end{theorem}

\begin{remark}[The need for a warm-up]
  The regret bound requires a warm-up. It can be eliminated by replacing the sampling matrix $G_t$ at each step with
\[
  \hat G_t = \lambda I + \sum_{j=1}^t \tilde{\nu}(\langle \hat\theta_{t}, X_j \rangle) X_j X_j\tran.
\]
However, the matrix $\hat G_t$ admits no rank-one update, and we consider it essential that the algorithm only uses quantities that can be computed in a streaming, online manner. Recent promising work on upper confidence bound-based approaches to generalised linear bandits provide online mirror descent-based ellipsoids \citep{zhang2026generalized}; however, these appear too weak to make our Thompson sampling proof go through.
\end{remark}

\section{Proof of main result}

For $t>\tau$, write
\[
  u_\star=\langle x_\star,\theta_\star\rangle,\qquad
  u_t^\star=\langle X_t,\theta_\star\rangle,\qquad
  \hat u_t=\langle X_t,\hat\theta_{t-1}\rangle,\qquad
  u_t=\langle X_t,\theta_t\rangle.
\]
We also write
\[
  \tilde u_t=(-B)\vee \hat u_t\wedge B,
  \qquad
  \tilde\nu_t=\nu(\tilde u_t),
  \qquad t>\tau,
\]
and define for all $t \geq 1$,
\[
  V_{t-1} = \lambda I + \sum_{s=1}^{t-1} X_s X_s\tran.
\]

The next lemma, established over the course of \cref{appendix:glm-warmup}, gives guarantees on the localisation of the above-defined quantities after the warm-up.

\begin{lemma}[name=Warm-up guarantee,restate=warmupLemma]\label{lem:warm-up}
Let $\gamma>0$. Assume that
\[
  \max_{x\in\cX}\norm{x}_{V_\tau^{-1}}^2
  \le
  \epsloc.
\]
Then there exists an event $\cE_{\rm loc}\subset\cE^\star$ such that $\P{\cE_{\rm loc}}\ge1-2\delta$ and, on
$\cE_{\rm loc}$,
\[
  \max_{1\le s\le n}\norm{\eta_s}_2\le \Gamma_\delta,
\]
and for every $t>\tau$,
\[
  u_t\in[u_\star-r_{\rm loc},u_\star+r_{\rm loc}],
  \qquad
  u_t^\star\in[u_\star-2r_{\rm loc},u_\star],
\]
while
\[
  |\hat u_t-u_t^\star|\le \frac{r_{\rm loc}}{2},
  \qquad
  |u_t-\hat u_t|\le \frac{r_{\rm loc}}{2}.
\]
\end{lemma}

In particular, on the above event supplied by the preceeding lemma (\cref{lem:warm-up}), the sampled natural parameters
$u_t$ and $u_t^\star$ lie in $[-B,B]$ for every $t>\tau$. The fitted natural parameter $\hat u_t$ is used only through its clipped version $\tilde u_t\in[-B,B]$.

On that event, we have the following decomposition
\[
\begin{aligned}
R_n
&\le
\tau\Delta
+\sum_{t=\tau+1}^n\roundb[\big]{\mu(u_\star)-\mu(u_t)}
+\sum_{t=\tau+1}^n\roundb[\big]{\mu(u_t)-\mu(u_t^\star)}.
\end{aligned}
\]
The first term is the deterministic warm-up cost, the middle term is the optimism deficit, and the last term is the sampling error.

We will repeatedly use the following elementary consequence of the definition of $M$.

\begin{proposition}[Self-concordant comparisons]\label{lem:sc}
For every $v,w\in[-B,B]$,
\[
  e^{-M|v-w|}\nu(w)\le\nu(v)\le e^{M|v-w|}\nu(w).
\]
Consequently, if $u_\star\in[-b,b]$ and
\[
  |u-u_\star|\vee |u'-u_\star|\le 2r_{\rm loc},
\]
then
\[
  |\mu(u)-\mu(u')|
  \le
  e^2\nu_\star |u-u'|.
\]
\end{proposition}

\begin{proof}
Since $\nu=\psi_Q''$ and $\nu'=\psi_Q'''$, the definition of $M$ gives
\[
  \abs*{\frac{\dif}{\dif u}\log\nu(u)}
  =
  \abs*{\frac{\nu'(u)}{\nu(u)}}
  \le M
\]
on $[-B,B]$. Integrating this display between $v$ and $w$ proves the variance comparison. For the second claim, let
$u_s=(1-s)u'+su$. Then $u_s\in[-B,B]$, since $2r_{\rm loc}\le B-b$, and
$\nu(u_s)\le e^{2Mr_{\rm loc}}\nu_\star\le e^2\nu_\star$. Hence
\[
  \mu(u)-\mu(u')
  =
  (u-u')\int_0^1\nu(u_s)\dif s,
\]
which gives the claim after taking absolute values.
\end{proof}

The following is a weighted, warm-start specialization of the standard elliptical potential lemma
\citep[Lemma~19.4]{lattimore2020bandit}.

\begin{proposition}[Elliptical potential after warm-up]\label{lem:telescope}
Assume that
\[
  \max_{x\in\cX}\norm{x}_{V_\tau^{-1}}^2\le\frac1{\kappa L}.
\]
On the event $\cE_{\rm loc}$ of \cref{lem:warm-up},
\[
  \sum_{t=\tau+1}^n\norm{X_t}_{G_{t-1}^{-1}}^2
  \le
  \frac{2e^2d}{\nu_\star}
  \log\roundb[\bigg]{1+\frac{nL}{\lambda d}}.
\]
\end{proposition}

\begin{proof}
On $\cE_{\rm loc}$, for $t>\tau$, since projection onto $[-B,B]$ is non-expansive,
\[
  |\tilde u_t - u_\star| \leq |\hat u_t-u_\star|
  \le
  |\hat u_t-u_t|+|u_t-u_\star|
  \le
  \frac{r_{\rm loc}}{2}+r_{\rm loc}
  \le
  \frac{3r_{\rm loc}}{2}.
\]
Hence \cref{lem:sc} gives
\[
  \tilde\nu_t=\nu(\tilde u_t)\ge e^{-2}\nu_\star.
\]
Therefore
\[
  \sum_{t=\tau+1}^n\norm{X_t}_{G_{t-1}^{-1}}^2
  =
  \sum_{t=\tau+1}^n
  \frac{1}{\tilde\nu_t}\norm{\sqrt{\tilde\nu_t}X_t}_{G_{t-1}^{-1}}^2
  \le
  \frac{e^2}{\nu_\star}
  \sum_{t=\tau+1}^n
  \norm{\sqrt{\tilde\nu_t}X_t}_{G_{t-1}^{-1}}^2.
\]

For $s\le\tau$, put $a_s=X_s/\sqrt\kappa$, and for $s>\tau$, put $a_s=\sqrt{\tilde\nu_s}X_s$. Then
\[
  G_t=\lambda I+\sum_{s=1}^t a_sa_s\tran.
\]
Also $\norm{a_s}_2^2\le L$ for all $s\le n$, by the definitions of $L,\kappa$ and the clipped curvature.

For $t>\tau$, since $\tilde\nu_s\ge1/\kappa$ and $\kappa\ge1$,
\[
  G_{t-1}\succeq \frac1\kappa V_{t-1}.
\]
The design condition and monotonicity of $V_t^{-1}$ give
\[
  \max_{x\in\cX}\norm{x}_{V_{t-1}^{-1}}^2
  \le
  \max_{x\in\cX}\norm{x}_{V_\tau^{-1}}^2
  \le
  \frac1{\kappa L}.
\]
Therefore,
\[
  \norm{a_t}_{G_{t-1}^{-1}}^2
  =
  \tilde\nu_t\norm{X_t}_{G_{t-1}^{-1}}^2
  \le
  L\kappa\norm{X_t}_{V_{t-1}^{-1}}^2
  \le1.
\]
Using $\log(1+z)\ge z/2$ for $z\in[0,1]$ and the determinant identity for rank-one updates (see proof of \citet[Lemma 19.4]{lattimore2020bandit} for details),
\[
  \sum_{t=\tau+1}^n\norm{a_t}_{G_{t-1}^{-1}}^2
  \le
  2\sum_{t=\tau+1}^n
  \log\roundb[\big]{1+\norm{a_t}_{G_{t-1}^{-1}}^2}
  \le
  2\log\frac{\det(G_n)}{\det(G_\tau)}
  \le
  2\log\frac{\det(G_n)}{\det(\lambda I)}.
\]
Finally, $\tr(G_n)\le\lambda d+nL$, and the arithmetic--geometric mean inequality gives
\[
  \log\frac{\det(G_n)}{\det(\lambda I)}
  \le
  d\log\roundb[\bigg]{1+\frac{nL}{\lambda d}}.
\]
(See note 20.2.1 of \citet{lattimore2020bandit} for details). Combining the displays proves the claim.
\end{proof}

We are ready to knock out the sampling term; this is the easy part, since both $u_t^\star$ and $u_t$ are based on the sampled arm $X_t$, and we need not compare with the optimal arm $x_\star$.

\begin{lemma}[Sampling-error bound]\label{lem:easy-term}
Assume that
\[
  \max_{x\in\cX}\norm{x}_{V_\tau^{-1}}^2\le\frac1{\kappa L}.
\]
On the event $\cE_{\rm loc}$ of \cref{lem:warm-up},
\[
  \sum_{t=\tau+1}^n\roundb[\big]{\mu(u_t)-\mu(u_t^\star)}
  \le
  e^3(\gamma\Gamma_\delta+\sqrt e)\beta_{n-1}^\delta
  \sqrt{2\nu_\star nd\log\roundb[\bigg]{1+\frac{nL}{\lambda d}}}.
\]
\end{lemma}

\begin{proof}
Fix $t>\tau$. On $\cE_{\rm loc}$, $|u_t-u_\star|\le r_{\rm loc}$ and
$|u_t^\star-u_\star|\le 2r_{\rm loc}$. By \cref{lem:sc},
\[
  \mu(u_t)-\mu(u_t^\star)
  \le
  e^2\nu_\star |u_t-u_t^\star|
  \le
  e^2\nu_\star
  \norm{X_t}_{G_{t-1}^{-1}}
  \norm{\theta_t-\theta_\star}_{G_{t-1}}.
\]
By the warm-up guarantee, for every $s<t$ with $s>\tau$, $|\hat u_s-u_s^\star|\le r_{\rm loc}/2$. Since projection onto
$[-B,B]$ is non-expansive and $u_s^\star\in[-B,B]$,
\[
  |\tilde u_s-u_s^\star|\le|\hat u_s-u_s^\star|\le\frac{r_{\rm loc}}{2}.
\]
Thus $\tilde\nu_s\le e\nu(u_s^\star)$. Together with $1/\kappa\le\nu(u_s^\star)$ for $s\le\tau$, this gives
\[
  G_{t-1}
  \preceq
  e\roundb[\Big]{\lambda I+\sum_{s=1}^{t-1}\nu(u_s^\star)X_sX_s\tran}
  =
  eH_{t-1}^\star.
\]
Therefore, on $\cE_{\rm loc}$,
\[
\begin{aligned}
  \norm{\theta_t-\theta_\star}_{G_{t-1}}
  &\le
  \norm{\theta_t-\hat\theta_{t-1}}_{G_{t-1}}
  +
  \norm{\hat\theta_{t-1}-\theta_\star}_{G_{t-1}}
  \\
  &\le
  \gamma\beta_{t-1}^\delta\norm{\eta_t}_2
  +
  \sqrt e\,\norm{\hat\theta_{t-1}-\theta_\star}_{H_{t-1}^\star}
  \\
  &\le
  \beta_{n-1}^\delta(\gamma\Gamma_\delta+\sqrt e).
\end{aligned}
\]
Summing and applying Cauchy--Schwarz,
\[
\begin{aligned}
  \sum_{t=\tau+1}^n\roundb[\big]{\mu(u_t)-\mu(u_t^\star)}
  &\le
  e^2\nu_\star\beta_{n-1}^\delta(\gamma\Gamma_\delta+\sqrt e)
  \sum_{t=\tau+1}^n\norm{X_t}_{G_{t-1}^{-1}}
  \\
  &\le
  e^2\nu_\star\beta_{n-1}^\delta(\gamma\Gamma_\delta+\sqrt e)
  \sqrt{
    n\sum_{t=\tau+1}^n\norm{X_t}_{G_{t-1}^{-1}}^2
  }.
\end{aligned}
\]
The result follows from \cref{lem:telescope}.
\end{proof}

\begin{lemma}[Gaussian Poincar\'e bound for support functions]
\label{lem:support-poincare}
Let $\cX\subset\Bd$ be compact and define
\[
  J(\theta)=\max_{x\in\cX}\langle x,\theta\rangle .
\]
Let $\eta\sim\cN(0,I_d)$, let $\bar\theta\in\Rd$, and let $A\in\R^{d\times d}$ be
invertible. Set
\[
  \theta=\bar\theta+A\eta,
  \qquad
  X(\theta)\in\argmax_{x\in\cX}\langle x,\theta\rangle,
\]
where $X(\theta)$ is any measurable maximiser. Then
\[
  \Var\!\left(J(\theta)\right)
  \le
  \E \norm{A\tran X(\theta)}_2^2 .
\]
\end{lemma}

\begin{proof}
The function $J$ is convex and $1$-Lipschitz. By Rademacher's theorem,
$J$ is differentiable Lebesgue-a.e. \citep[Theorem~3.2]{evansgariepy2015measure}.
By Danskin's theorem,
\[
  \partial J(\theta)
  =
  \conv\argmax_{x\in\cX}\langle x,\theta\rangle .
\]
Moreover, a convex function is differentiable at $\theta$ if and only if its
subdifferential at $\theta$ is a singleton
\citep[Props.~B.21(c) and~B.22(b)]{bertsekas2016nonlinear}. Hence, at every
differentiability point of $J$, every maximiser in
$\argmax_{x\in\cX}\langle x,\theta\rangle$ is equal to $\nabla J(\theta)$.

Since $A$ is invertible, the law of $\theta=\bar\theta+A\eta$ is absolutely
continuous with respect to Lebesgue measure. Thus the nondifferentiability set
of $J$ has probability zero under the law of $\theta$. Consequently, for
Gaussian-a.e. $\eta$, the function
\[
  f(\eta)=J(\bar\theta+A\eta)
\]
is differentiable and satisfies
\[
  \nabla f(\eta)
  =
  A\tran X(\theta).
\]

Applying the Gaussian Poincar\'e inequality to $f$, using the standard
Lipschitz approximation argument by convolution
\citep[Theorem~3.20 and the following approximation paragraph]{boucheron2013concentration},
gives
\[
  \Var(f(\eta))
  \le
  \E\norm{\nabla f(\eta)}_2^2
  =
  \E\norm{A\tran X(\theta)}_2^2 .
\]
This proves the claim.
\end{proof}

The proof of the next proposition is in \cref{appendix:glm-optimism}.

\begin{proposition}[name=Optimism,restate=optimismProp]\label{prop:optimism}
Fix $p\in(0,1/2)$. Assume that
\[
  \max_{x\in\cX}\norm{x}_{V_\tau^{-1}}^2
  \le
  \frac{r_{\rm loc}^2}{\kappa(\beta_{n-1}^\delta)^2}.
\]
Let $z_p \in \R$ be such that if
$Z \sim \cN(0,1)$, then $\P{Z\ge z_p}=p$. If $\gamma\ge\sqrt e/z_p$, then on $\cE^\star$, for every $t>\tau$,
\[
  \P{u_t\ge u_\star\mid \cF_{t-1}}\ge p.
\]
\end{proposition}

\begin{lemma}[Cantelli shortfall bound]\label{lem:cantelli}
Let $X$ be square-integrable with variance $\sigma^2$. If $\P{X\ge a}\ge p$ for some $p\in(0,1)$, then
\[
  \E[(a-X)_+]\le
  \roundb[\Big]{1+\sqrt{\frac{1-p}{p}}}\sigma.
\]
\end{lemma}

\begin{proof}
Let $m_X=\E X$ and $r=(a-m_X)_+$. If $r>0$, then $\{X\ge a\}\subseteq\{X-m_X\ge r\}$, so Cantelli's inequality \citep[Exercise~2.3]{boucheron2013concentration} gives
\[
  p\le \P{X-m_X\ge r}\le \frac{\sigma^2}{\sigma^2+r^2}.
\]
Thus $r\le\sigma\sqrt{(1-p)/p}$. Since
\[
  (a-X)_+\le(a-m_X)_+ +(m_X-X)_+,
\]
taking expectations and using $\E(m_X-X)_+\le\E|X-m_X|\le\sigma$ proves the claim.
\end{proof}

The following predictable-to-realised conversion is a standard
extension of Freedman's martingale Bernstein inequality.

\begin{lemma}[Freedman inequality]\label{lem:freedman}
Let $(Z_t)_{t=1}^n$ be a nonnegative adapted sequence with respect to
$(\cF_t)_{t=0}^n$. Assume that $0\le Z_t\le b$ almost surely for a deterministic
$b>0$. Define
\[
  m_t=\E_{t-1}Z_t,\qquad
  S=\sum_{t=1}^n Z_t,\qquad
  T=\sum_{t=1}^n m_t.
\]
Then, for every $\rho\in(0,1)$ and every $\delta\in(0,1)$, with probability at
least $1-\delta$,
\[
  T
  \le
  \frac{1}{1-\rho}S
  +
  \frac{b}{\rho(1-\rho)}\log\frac1\delta .
\]
\end{lemma}

\begin{proof}
Let $Y_t=m_t-Z_t$. Then $(Y_t)_{t=1}^n$ is a martingale difference sequence and
$Y_t\le m_t\le b$. Moreover,
\[
  \E_{t-1}Y_t^2
  =
  \E_{t-1}(Z_t-m_t)^2
  \le
  \E_{t-1}Z_t^2
  \le
  b\,m_t.
\]
By the Freedman bound \citep{freedman1975tail} as stated in \citet[Exercise~5.15]{lattimore2020bandit}, for any fixed
$\eta\in(0,1/b]$, with probability at least $1-\delta$,
\[
  \sum_{t=1}^n Y_t
  \le
  \eta\sum_{t=1}^n\E_{t-1}Y_t^2
  +
  \frac1\eta\log\frac1\delta .
\]
Using the variance bound and choosing $\eta=\rho/b$ gives
\[
  T-S
  \le
  \rho T
  +
  \frac{b}{\rho}\log\frac1\delta .
\]
Rearranging proves the claim.
\end{proof}
\begin{lemma}[Optimism-deficit bound]\label{lem:hard-term}
Fix $p\in(0,1/2)$. Assume $\gamma\ge\sqrt e/z_p$ and
\[
  \max_{x\in\cX}\norm{x}_{V_\tau^{-1}}^2
  \le
  \frac1{\kappa L}
  \wedge
  \frac{r_{\rm loc}^2}{\kappa(\beta_{n-1}^\delta)^2}.
\]
There exist events $\cE_A$ and $\cE_F$, each of probability at least $1-\delta$, such that on
$\cE_{\rm loc}\cap\cE_A\cap\cE_F$,
\[
\begin{aligned}
\sum_{t=\tau+1}^n\roundb[\big]{\mu(u_\star)-\mu(u_t)}
\le{}&
e\nu_\star C_p\gamma\beta_{n-1}^\delta\sqrt n
\sqrt{
  \frac{16}{3\lambda}\log(1/\delta)
  +
  \frac{8e^2d}{3\nu_\star}
  \log\roundb[\bigg]{1+\frac{nL}{\lambda d}}
}
\\
&\quad
+e\nu_\star r_{\rm loc}\sqrt{2n\log(1/\delta)}.
\end{aligned}
\]
\end{lemma}

\begin{proof}
Work on $\cE_{\rm loc}$ and fix $t>\tau$. Since $\mu$ is increasing,
\[
  \mu(u_\star)-\mu(u_t)
  \le
  \roundb[\big]{\mu(u_\star)-\mu(u_t)}\1{u_\star\ge u_t}.
\]
Let
\[
  Z_t=(u_\star-u_t)_+,\qquad
  \bar Z_t=Z_t\wedge r_{\rm loc}.
\]
On $\cE_{\rm loc}$, $u_t\in[u_\star-r_{\rm loc},u_\star+r_{\rm loc}]$, so $Z_t=\bar Z_t$. Moreover, by \cref{lem:sc},
\[
  \roundb[\big]{\mu(u_\star)-\mu(u_t)}\1{u_\star\ge u_t}
  \le
  e\nu_\star \bar Z_t.
\]

Apply Azuma--Hoeffding to the martingale
\[
  \sum_{s=1}^t\1{s>\tau}\roundb[\big]{\bar Z_s-\E_{s-1}\bar Z_s}.
\]
The increments are bounded by $r_{\rm loc}$, so there is an event $\cE_A$, with $\P{\cE_A}\ge1-\delta$, such that
\[
  \sum_{t=\tau+1}^n\bar Z_t
  \le
  \sum_{t=\tau+1}^n\E_{t-1}\bar Z_t
  +
  r_{\rm loc}\sqrt{2n\log(1/\delta)}.
\]
Hence, on $\cE_{\rm loc}\cap\cE_A$,
\[
  \sum_{t=\tau+1}^n\roundb[\big]{\mu(u_\star)-\mu(u_t)}
  \le
  e\nu_\star\sum_{t=\tau+1}^n\E_{t-1}Z_t
  +
  e\nu_\star r_{\rm loc}\sqrt{2n\log(1/\delta)}.
\]

By \cref{prop:optimism,lem:cantelli}, on $\cE^\star$, writing $C_p = 1 + \sqrt{\frac{1-p}{p}}$,
\[
  \E_{t-1}Z_t
  =
  \E_{t-1}(u_\star-u_t)_+
  \le
  C_p\sqrt{\Var_{t-1}(u_t)}.
\]

Let $J(\theta)=\max_{x\in\cX}\langle x,\theta\rangle$. Then $u_t=J(\theta_t)$.
Conditionally on $\cF_{t-1}$, apply \cref{lem:support-poincare} with
\[
  \bar\theta=\hat\theta_{t-1},
  \qquad
  A=\gamma\beta_{t-1}^\delta G_{t-1}^{-1/2}.
\]
The matrix $A$ is invertible because $\gamma>0$, $\beta_{t-1}^\delta>0$, and
$G_{t-1}\succ0$. Hence
\[
  \Var_{t-1}(u_t)
  \le
  \gamma^2(\beta_{t-1}^\delta)^2
  \E_{t-1}\norm{X_t}_{G_{t-1}^{-1}}^2 .
\]

For $1\le t\le n$, define
\[
  \ell_t=\1{t>\tau}\norm{X_t}_{G_{t-1}^{-1}}^2,\qquad
  m_t=\E_{t-1}\ell_t.
\]
Then
\[
  \sum_{t=\tau+1}^n\E_{t-1}Z_t
  \le
  C_p\gamma\beta_{n-1}^\delta\sum_{t=1}^n\sqrt{m_t}
  \le
  C_p\gamma\beta_{n-1}^\delta\sqrt{nT},
\]
where $T=\sum_{t=1}^n m_t$. Let $S=\sum_{t=1}^n\ell_t$. Since $G_{t-1}\succeq\lambda I$ and $\cX\subset\Bd$,
\[
  0\le \ell_t\le\frac1\lambda.
\]
Applying \cref{lem:freedman} with $Z_t=\ell_t$, $b=1/\lambda$, and $\rho=1/4$, there is an event $\cE_F$, with
$\P{\cE_F}\ge1-\delta$, such that
\[
  T
  \le
  \frac{4}{3}S
  +
  \frac{16}{3\lambda}\log(1/\delta).
\]
On $\cE_{\rm loc}$, \cref{lem:telescope} gives
\[
  S
  \le
  \frac{2e^2d}{\nu_\star}
  \log\roundb[\bigg]{1+\frac{nL}{\lambda d}}.
\]
Therefore, on $\cE_{\rm loc}\cap\cE_F$,
\[
  \sqrt T
  \le
  \sqrt{
    \frac{16}{3\lambda}\log(1/\delta)
    +
    \frac{8e^2d}{3\nu_\star}
    \log\roundb[\bigg]{1+\frac{nL}{\lambda d}}
  }.
\]
Substituting this into the previous displays proves the claim.
\end{proof}

\begin{proof}[Proof of \cref{thm:glm-gaussian-sampling-regret}]
The warm-up length bound in \cref{prop:warmup-length} gives
\[
  \tau
  \le
  \frac{4d}{\epsloc}
  \log\roundb[\bigg]{1+\frac{4d}{\epsloc\lambda}}
  +1.
\]
If $\tau=n$, then $R_n\le n\Delta=\tau\Delta$, and the result follows. Hence we may assume $\tau<n$. In this case the warm-up phase has reached the design condition, so
\[
  \max_{x\in\cX}\norm{x}_{V_\tau^{-1}}^2
  \le
  \epsloc.
\]
If $\bar\beta_{n-1}^\delta=0$, then $\beta_{n-1}^\delta=0$. Since $(\beta_t^\delta)$ is nondecreasing, on $\cE^\star$ we have
$\hat\theta_{t-1}=\theta_\star$ for all $t\le n$, and in the sampling phase $\theta_t=\theta_\star$. Hence the post-warm-up regret is zero on $\cE^\star$, and the result follows. We may therefore assume $\bar\beta_{n-1}^\delta>0$.

The definition of $\epsloc$, together with $\beta_{n-1}^\delta\le\bar\beta_{n-1}^\delta$, gives
\[
  \epsloc
  \le
  \frac{r_{\rm loc}^2}{4\kappa(\beta_{n-1}^\delta)^2(1\vee\gamma\Gamma_\delta)^2}.
\]
Thus \cref{lem:warm-up} applies. The same design condition also implies the premises of
\cref{lem:telescope,lem:easy-term,lem:hard-term}. We apply the optimism results with $p=1/4$. Since
$z_{1/4}>1/2$, the choice $\gamma=4$ satisfies $\gamma\ge\sqrt e/z_{1/4}$.

On the event $\cE_{\rm loc}\cap\cE_A\cap\cE_F$, whose probability is at least $1-4\delta$, the regret decomposition gives
\[
  R_n
  \le
  \tau\Delta
  +
  \sum_{t=\tau+1}^n\roundb[\big]{\mu(u_\star)-\mu(u_t)}
  +
  \sum_{t=\tau+1}^n\roundb[\big]{\mu(u_t)-\mu(u_t^\star)}.
\]
The second term is bounded by \cref{lem:hard-term}, and the third by \cref{lem:easy-term}. Writing
$x_\delta=\log(1/\delta)$ and using $C_{1/4}=1+\sqrt3$, these two lemmas give
\[
\begin{aligned}
  R_n
  \le{}&
  \Delta\tau
  +4e(1+\sqrt3)\beta_{n-1}^\delta\nu_\star\sqrt n
  \sqrt{
    \frac{16x_\delta}{3\lambda}
    +
    \frac{8e^2d\ell_n}{3\nu_\star}
  }
  \\
  &\quad
  +e\sqrt2\,\nu_\star r_{\rm loc}\sqrt{nx_\delta}
  +e^3(4\Gamma_\delta+\sqrt e)\beta_{n-1}^\delta
  \sqrt{2\nu_\star nd\,\ell_n}.
\end{aligned}
\]
Using $\sqrt{a+b}\le\sqrt a+\sqrt b$ and $\Gamma_\delta\ge1$, the two terms involving
$\sqrt{\nu_\star nd\ell_n}$ are bounded by
\[
  \roundb[\Big]{4e^2\sqrt{\frac{8}{3}}(1+\sqrt3)+e^3\sqrt2(4+\sqrt e)}
  \beta_{n-1}^\delta\Gamma_\delta
  \sqrt{\nu_\star nd\,\ell_n}
  \le
  300\,\beta_{n-1}^\delta\Gamma_\delta
  \sqrt{\nu_\star nd\,\ell_n}.
\]
Similarly, the remaining terms are bounded by
\[
  \frac{16e(1+\sqrt3)}{\sqrt3}
  \beta_{n-1}^\delta\nu_\star
  \sqrt{\frac{nx_\delta}{\lambda}}
  +
  e\sqrt2\,\nu_\star r_{\rm loc}\sqrt{nx_\delta}
  \le
  80\,\nu_\star
  \roundb[\Big]{\frac{\beta_{n-1}^\delta}{\sqrt\lambda}+r_{\rm loc}}
  \sqrt{nx_\delta}.
\]
Combining the last two displays proves the regret bound.
\end{proof}

\printbibliography
\appendix

\clearpage
\section{Optimism}
\label{appendix:glm-optimism}

\optimismProp*

\begin{proof}
Fix $t>\tau$, and work on $\cE^\star$. If $x_\star=0$, then $0\in\cX$ and
\[
  u_t=\max_{x\in\cX}\langle\theta_t,x\rangle\ge0=u_\star,
\]
so the claim is immediate. Assume henceforth that $x_\star\ne0$.

By \cref{lem:warmup-mle} with $r=r_{\rm loc}$,
\[
  |\langle X_s,\hat\theta_{s-1}-\theta_\star\rangle|
  \le
  r_{\rm loc},
  \qquad
  \tau<s\le t-1.
\]
Since projection onto $[-B,B]$ is non-expansive and $u_s^\star\in[-B,B]$,
\[
  |\tilde u_s-u_s^\star|
  \le
  |\hat u_s-u_s^\star|
  \le
  r_{\rm loc},
  \qquad \tau<s\le t-1.
\]
Since $Mr_{\rm loc}\le1$, \cref{lem:sc} gives
\[
  \tilde\nu_s\le e\nu(u_s^\star),
  \qquad \tau<s\le t-1.
\]
For $s\le\tau$, $1/\kappa\le\nu(\langle\theta_\star,X_s\rangle)$. Therefore
\[
\begin{aligned}
  G_{t-1}
  &=
  \lambda I
  +\frac1\kappa\sum_{s=1}^{\tau}X_sX_s\tran
  +\sum_{s=\tau+1}^{t-1}\tilde\nu_sX_sX_s\tran
  \\
  &\preceq
  e\roundb[\Bigg]{
    \lambda I+\sum_{s=1}^{t-1}\nu(\langle\theta_\star,X_s\rangle)X_sX_s\tran
  }
  =
  eH_{t-1}^\star.
\end{aligned}
\]
Since $X_t$ maximises $\langle\theta_t,x\rangle$,
\[
\begin{aligned}
  u_t-u_\star
  &\ge
  \langle x_\star,\theta_t-\theta_\star\rangle
  \\
  &=
  \langle x_\star,\hat\theta_{t-1}-\theta_\star\rangle
  +
  \gamma\beta_{t-1}^\delta
  \langle x_\star,G_{t-1}^{-1/2}\eta_t\rangle.
\end{aligned}
\]
By Cauchy--Schwarz, the confidence event, and $G_{t-1}\preceq eH_{t-1}^\star$,
\[
  \langle x_\star,\hat\theta_{t-1}-\theta_\star\rangle
  \ge
  -\norm{x_\star}_{H_{t-1}^{\star-1}}
  \norm{\hat\theta_{t-1}-\theta_\star}_{H_{t-1}^\star}
  \ge
  -\sqrt e\,\beta_{t-1}^\delta\norm{x_\star}_{G_{t-1}^{-1}}.
\]
If $\beta_{t-1}^\delta\norm{x_\star}_{G_{t-1}^{-1}}=0$, the preceding display already gives $u_t\ge u_\star$. Otherwise define
\[
  Z_t
  =
  \frac{
    \langle x_\star,G_{t-1}^{-1/2}\eta_t\rangle
  }{
    \norm{x_\star}_{G_{t-1}^{-1}}
  }.
\]
Conditionally on $\cF_{t-1}$, $Z_t\sim\cN(0,1)$, and
\[
  u_t-u_\star
  \ge
  \beta_{t-1}^\delta\norm{x_\star}_{G_{t-1}^{-1}}
  \roundb[\big]{\gamma Z_t-\sqrt e}.
\]
Thus
\[
  \P{u_t\ge u_\star\mid\cF_{t-1}}
  \ge
  \P{Z_t\ge \sqrt e/\gamma\mid\cF_{t-1}}
  =
  \P{Z\ge \sqrt e/\gamma}
  \ge
  \P{Z\ge z_p}
  =
  p,
\]
where the penultimate inequality uses $\gamma\ge\sqrt e/z_p$.
\end{proof}

\clearpage
\section{Warm-up and localisation}
\label{appendix:glm-warmup}

The following argument is from Appendix B of \citet{janz2024evill}, itself based on Exercise 19.3 of \citet{lattimore2020bandit}.

\begin{proposition}[Greedy deterministic exploration]\label{prop:warmup-length}
Fix $a\in(0,1]$. Consider the deterministic routine that starts from $V_0=\lambda I$ and, while
$\max_{x\in\cX}\norm{x}_{V_t^{-1}}^2>a$, chooses
\[
  X_{t+1}\in\argmax_{x\in\cX}\norm{x}_{V_t^{-1}},
  \qquad
  V_{t+1}=V_t+X_{t+1}X_{t+1}\tran.
\]
Then the routine reaches $\max_{x\in\cX}\norm{x}_{V_t^{-1}}^2\le a$ after at most
\[
  \frac{4d}{a}
  \log\roundb[\bigg]{1+\frac{4d}{a\lambda}}
  +1
\]
rounds.
\end{proposition}

\begin{proof}
Set
\[
  m=
  \ceil*{
    \frac{4d}{a}
    \log\roundb[\bigg]{1+\frac{4d}{a\lambda}}
  }.
\]
If the routine has reached the design condition by time $m$, there is nothing to prove. Suppose instead that it has not. Then throughout rounds $1,\dots,m$ it chooses
\[
  X_s\in\argmax_{x\in\cX}\norm{x}_{V_{s-1}^{-1}},
\]
and the corresponding maximum is larger than $a$.

Let
\[
  g_t=\max_{x\in\cX}\norm{x}_{V_t^{-1}}^2.
\]
Since $V_t\succeq V_{t-1}$, the sequence $(g_t)$ is nonincreasing. Under the contradiction hypothesis,
$g_{s-1}>a$ for $s=1,\dots,m$, and hence
\[
  \norm{X_s}_{V_{s-1}^{-1}}^2=g_{s-1}>a.
\]
By the matrix determinant lemma,
\[
  \det(V_s)
  =
  \det(V_{s-1})
  \roundb[\big]{1+\norm{X_s}_{V_{s-1}^{-1}}^2}
  >
  (1+a)\det(V_{s-1}).
\]
Iterating,
\[
  \det(V_m)>\lambda^d(1+a)^m.
\]
On the other hand, $\cX\subset\Bd$ implies $\tr(V_m)\le\lambda d+m$, and so the arithmetic--geometric mean inequality yields
\[
  \det(V_m)
  \le
  \roundb[\Big]{\frac{\tr(V_m)}{d}}^d
  \le
  \roundb[\Big]{\lambda+\frac{m}{d}}^d.
\]
Combining the two determinant bounds,
\[
  m\log(1+a)
  <
  d\log\roundb[\bigg]{1+\frac{m}{\lambda d}}.
\]
Since $a\in(0,1]$, $\log(1+a)\ge a/2$, hence
\[
  \frac{ma}{2d}
  <
  \log\roundb[\bigg]{1+\frac{m}{\lambda d}}.
\]
Let $r=ma/(4d)$. By the definition of $m$,
\[
  r\ge \log\roundb[\bigg]{1+\frac{4d}{a\lambda}}.
\]
Since $d\ge1$,
\[
  \log\roundb[\bigg]{1+\frac{4}{a\lambda}}\le r.
\]
Therefore,
\[
  \log\roundb[\bigg]{1+\frac{m}{\lambda d}}
  =
  \log\roundb[\bigg]{1+\frac{4r}{a\lambda}}
  \le
  \log(1+r)
  +
  \log\roundb[\bigg]{1+\frac{4}{a\lambda}}
\le
  r+r
  =
  \frac{ma}{2d},
\]
contradicting the previous strict inequality. Hence the routine reaches the design condition by time $m$, and the claimed bound follows from
\[
  m\le
  \frac{4d}{a}
  \log\roundb[\bigg]{1+\frac{4d}{a\lambda}}
  +1.
\]
\end{proof}

\begin{lemma}\label{lem:warmup-mle}
Fix $r>0$. If
\[
  \max_{x\in\cX}\norm{x}_{V_\tau^{-1}}^2
  \le
  \frac{r^2}{\kappa(\beta_{n-1}^\delta)^2},
\]
with the right-hand side interpreted as $+\infty$ when $\beta_{n-1}^\delta=0$, then on $\cE^\star$, for every
$t>\tau$,
\[
  \sup_{x\in\cX}
  |\langle x,\hat\theta_{t-1}-\theta_\star\rangle|
  \le r.
\]
\end{lemma}

\begin{proof}
Fix $t>\tau$, and work on $\cE^\star$. If $\beta_{n-1}^\delta=0$, then
\[
  \norm{\hat\theta_{t-1}-\theta_\star}_{H_{t-1}^\star}=0,
\]
and the claim follows since $H_{t-1}^\star$ is positive definite. Hence assume $\beta_{n-1}^\delta>0$.

By monotonicity of $V_s^{-1}$,
\[
  \max_{x\in\cX}\norm{x}_{V_{t-1}^{-1}}^2
  \le
  \max_{x\in\cX}\norm{x}_{V_\tau^{-1}}^2.
\]
Since $|\langle\theta_\star,X_s\rangle|\le b$, we have $\nu(\langle\theta_\star,X_s\rangle)\ge1/\kappa$. Also $\kappa\ge1$, so
\[
  H_{t-1}^\star
  =
  \lambda I+\sum_{s=1}^{t-1}\nu(\langle\theta_\star,X_s\rangle)X_sX_s\tran
  \succeq
  \frac1\kappa
  \roundb[\Big]{\lambda I+\sum_{s=1}^{t-1}X_sX_s\tran}
  =
  \frac1\kappa V_{t-1}.
\]
Thus $H_{t-1}^{\star-1}\preceq\kappa V_{t-1}^{-1}$. For any $x\in\cX$,
\[
  |\langle x,\hat\theta_{t-1}-\theta_\star\rangle|
  \le
  \norm{x}_{H_{t-1}^{\star-1}}
  \norm{\hat\theta_{t-1}-\theta_\star}_{H_{t-1}^\star}
 \le
  \beta_{n-1}^\delta\sqrt\kappa\norm{x}_{V_{t-1}^{-1}}
  \le
  \beta_{n-1}^\delta\sqrt{\kappa}
  \sup_{z\in\cX}\norm{z}_{V_\tau^{-1}}
  \le r.
\]
Taking the supremum over $x\in\cX$ proves the claim.
\end{proof}

\begin{proposition}\label{prop:warmup-gaussian}
The event
\[
  \cE_\eta=
  \curlyb[\Big]{
    \max_{1\le t\le n}\norm{\eta_t}_2\le \Gamma_\delta
  }
\]
satisfies $\P{\cE_\eta}\ge1-\delta$.
\end{proposition}

\begin{proof}
For notational convenience, take $\eta_1,\dots,\eta_n$ to be sampled at the outset, and use only those with $t>\tau$.
For $G\sim\cN(0,I_d)$, the standard Gaussian norm bound gives
\[
  \P{\norm{G}_2\ge\sqrt d+\sqrt{2x}}\le e^{-x},
  \qquad x>0.
\]
Taking $x=\log(n/\delta)$ gives $\P{\norm{G}_2\ge\Gamma_\delta}\le\delta/n$. The claim follows by a union bound over
$t=1,\dots,n$.
\end{proof}

\begin{lemma}\label{lem:warmup-perturbation}
Fix $r>0$. If
\[
  \max_{x\in\cX}\norm{x}_{V_\tau^{-1}}^2
  \le
  \frac{r^2}{\kappa\gamma^2(\beta_{n-1}^\delta)^2\Gamma_\delta^2},
\]
with the right-hand side interpreted as $+\infty$ when $\beta_{n-1}^\delta=0$, then on $\cE_\eta$, for every
$t>\tau$,
\[
  \sup_{x\in\cX}
  |\langle x,\theta_t-\hat\theta_{t-1}\rangle|
  \le r.
\]
\end{lemma}

\begin{proof}
Fix $t>\tau$, $x\in\cX$, and work on $\cE_\eta$. If $\beta_{n-1}^\delta=0$, then
$\theta_t=\hat\theta_{t-1}$ and the claim is immediate. Hence assume $\beta_{n-1}^\delta>0$.

By monotonicity of $V_s^{-1}$,
\[
  \max_{x\in\cX}\norm{x}_{V_{t-1}^{-1}}^2
  \le
  \max_{x\in\cX}\norm{x}_{V_\tau^{-1}}^2.
\]
By the definition of the clipped curvature,
\[
  G_{t-1}
  \succeq
  \frac1\kappa V_{t-1}.
\]
Thus $G_{t-1}^{-1}\preceq\kappa V_{t-1}^{-1}$. Since
\[
  \theta_t-\hat\theta_{t-1}
  =
  \gamma\beta_{t-1}^\delta G_{t-1}^{-1/2}\eta_t,
\]
we obtain
\[
  |\langle x,\theta_t-\hat\theta_{t-1}\rangle|
  \le
  \gamma\beta_{t-1}^\delta
  \norm{x}_{G_{t-1}^{-1}}
  \norm{\eta_t}_2
  \le
  \gamma\beta_{n-1}^\delta\Gamma_\delta\sqrt\kappa\norm{x}_{V_{t-1}^{-1}}
  \le
  \gamma\beta_{n-1}^\delta\Gamma_\delta\sqrt\kappa
  \sup_{z\in\cX}\norm{z}_{V_\tau^{-1}}
  \le r.
\]
Taking the supremum over $x\in\cX$ proves the claim.
\end{proof}

\warmupLemma*

\begin{proof}
Let
\[
  \cE_{\rm loc}=\cE^\star\cap\cE_\eta.
\]
By \cref{prop:warmup-gaussian}, $\P{\cE_{\rm loc}}\ge1-2\delta$, and the Gaussian norm bound in the statement holds by definition of $\cE_\eta$.

Work on $\cE_{\rm loc}$, and fix $t>\tau$. The assumed design condition implies both
\[
  \max_{x\in\cX}\norm{x}_{V_\tau^{-1}}^2
  \le
  \frac{(r_{\rm loc}/2)^2}{\kappa(\beta_{n-1}^\delta)^2},
  \qquad
  \max_{x\in\cX}\norm{x}_{V_\tau^{-1}}^2
  \le
  \frac{(r_{\rm loc}/2)^2}{\kappa\gamma^2(\beta_{n-1}^\delta)^2\Gamma_\delta^2},
\]
again with the convention that the right-hand sides are $+\infty$ when $\beta_{n-1}^\delta=0$. Hence
\cref{lem:warmup-mle,lem:warmup-perturbation} apply with $r=r_{\rm loc}/2$. Therefore
\[
  \sup_{x\in\cX}
  |\langle x,\hat\theta_{t-1}-\theta_\star\rangle|
  \le
  \frac{r_{\rm loc}}{2},
  \qquad
  \sup_{x\in\cX}
  |\langle x,\theta_t-\hat\theta_{t-1}\rangle|
  \le
  \frac{r_{\rm loc}}{2}.
\]
Taking $x=X_t$ gives
\[
  |\hat u_t-u_t^\star|\le \frac{r_{\rm loc}}{2},
  \qquad
  |u_t-\hat u_t|\le \frac{r_{\rm loc}}{2}.
\]
Also,
\[
  \sup_{x\in\cX}|\langle x,\theta_t-\theta_\star\rangle|\le r_{\rm loc}.
\]
Since $X_t$ maximises $\langle\theta_t,x\rangle$ over $\cX$,
\[
  u_t
  =
  \langle X_t,\theta_t\rangle
  \ge
  \langle x_\star,\theta_t\rangle
  =
  u_\star+\langle x_\star,\theta_t-\theta_\star\rangle
  \ge
  u_\star-r_{\rm loc}.
\]
On the other hand,
\[
  u_t
  =
  u_t^\star+\langle X_t,\theta_t-\theta_\star\rangle
  \le
  u_\star+r_{\rm loc},
\]
because $u_t^\star\le u_\star$. Finally,
\[
  u_t^\star
  =
  u_t+\langle X_t,\theta_\star-\theta_t\rangle
  \ge
  u_t-r_{\rm loc}
  \ge
  u_\star-2r_{\rm loc}.
\]
This proves all claims.
\end{proof}

\end{document}